\documentclass[twoside]{article}
\usepackage[accepted]{aistats2015}
\pdfoutput=1

\usepackage{graphicx}
\usepackage{amsmath,amssymb}
\usepackage{bm}
\usepackage{appendix}
\usepackage[shortlabels]{enumitem}

\newcommand\blfootnote[1]{%
  \begingroup
  \renewcommand\thefootnote{}\footnote{#1}%
  \addtocounter{footnote}{-1}%
  \endgroup
}

\usepackage{amsthm}

\newtheorem{theorem}{Theorem}[section]
\newtheorem{lemma}[theorem]{Lemma}

\newtheorem{corollary}[theorem]{Corollary}

\newcommand{\pq}[1]{\left( #1 \right)}

\newcommand{\N}[1]{ {\| #1 \|}}

\newcommand{\numberthis}{\addtocounter{equation}{1}\tag{\theequation}}

\usepackage{color}

\newcommand{\Moment}{\mathbf{M}}
\newcommand{\SecM}{\Moment_2}
\newcommand{\SecMH}{\hat{\Moment}_2}
\newcommand{\RM}{\mathbf{R}}
\newcommand{\RMij}[2]{\RM_{#1#2}}
\newcommand{\deltaR}{\delta_{\RM}}
\newcommand{\SVK}[2]{\sigma_{#1}(#2)}
\newcommand{\EVK}[2]{\lambda_{#1}(#2)}

\newcommand{\Expect}[1]{\mathbb{E}[#1]}
\newcommand{\ExpectC}[2]{\mathbb{E}[#1|#2]}

\newcommand{\ExpectTC}[2]{\mathbb{E}^2[#1|#2]}
\newcommand{\FNorm}[1]{||#1||_{\textsc{F}}}
\newcommand{\TNorm}[1]{||#1||_2}

\begin{document}

%

%

\twocolumn[

\aistatstitle{Model Selection for Topic Models via Spectral Decomposition}

\aistatsauthor{Dehua Cheng\textsuperscript{*} \And  Xinran He\textsuperscript{*} \And Yan Liu}

\aistatsaddress{
		  dehua.cheng@usc.edu \\ University of Southern California
 \And     xinranhe@usc.edu 	  \\ University of Southern California
 \And 	  yanliu.cs@usc.edu    \\ University of Southern California
  } ]


\vspace{1.1in}

\begin{abstract}
Topic models have achieved significant successes in analyzing large-scale text corpus. In practical applications, we are always confronted with the challenge of model selection, i.e., how to appropriately set the number of topics.  Following recent advances in topic model inference  via tensor decomposition, we make a first attempt to provide theoretical analysis on model selection in {\emph{latent Dirichlet allocation}}.
Under mild conditions, we derive the upper bound and lower bound on the number of topics given a text collection of finite size. Experimental results demonstrate that our bounds are accurate and tight. Furthermore, using \emph{Gaussian mixture model} as an example, we show that our methodology can be easily generalized to model selection analysis for other latent models.
\end{abstract}

\section{Introduction}

Recently topic models, such as latent Dirichlet allocation (LDA)~\cite{blei2003latent} and its variants \cite{teh2006hierarchical}, have been proven extremely successful in modeling large, complex text corpus.
These models assume that the words in a document are generated from a mixture of  latent topics represented by multinomial distributions over a given dictionary.
 Therefore, the major inference problem becomes recovering latent topics from text corpus.
Popular inference algorithms for LDA include variational inference~\cite{blei2003latent,teh2007collapsed,wang2011online,hoffman2013stochastic},
 sampling methods~\cite{Griffiths2004,porteous2008fast},
  and tensor decomposition~\cite{arora2012learning,anandkumar2012spectral,Anima2012tensor} recently.
   However, all of them require that
   the number of topics $K$ is given as input.

 It is known that model selection, i.e., choosing the appropriate number of topics $K$ plays a vital role in successfully applying LDA models \cite{Jian2014topicNumber, kulesza2014lowrank}.
    For example, ~\cite{Jian2014topicNumber} has shown that a large value of $K$ leads to severe deterioration in the learning rate; ~\cite{kulesza2014lowrank} points out that
     incorrect number of mixture components
      can result in an unpredictable error when estimating parameters of mixture model 
 via      
       spectral {methods}. Moreover, as $K$ increases, the computational {cost of inference for the} LDA model grows significantly.\blfootnote{\textsuperscript{*}Dehua Cheng and Xinran He contributed equally to this article.}

Unfortunately, it is extremely challenging to choose the number of topics for the LDA model.
   In practice,~\cite{AISTATS2012_Taddy12} approximates the marginal likelihood via Laplace's method, while~\cite{Airoldi07122010,Griffiths2004} computes the likelihood via MCMC. Moreover,~\cite{AISTATS2012_Taddy12} proposes another model selection method by analysis of residuals. However, it only provides rough measures for evidence in favor of a larger $K$. Other model selection criteria, such as AIC~\cite{akaike1974new}, BIC~\cite{schwarz1978estimating} and cross validation can be applied. Though achieving practical success~\cite{Airoldi07122010}, they only have asymptotic model selection consistency.  Moreover, they require multiple runs of the learning algorithm with 
   a wide range of  
   $K$, which limits 
 its   
    practicality on large-scale datasets. Bayesian nonparametrics, such as \emph{Hierarchical Dirichlet Processes}~(HDP)\cite{teh2006hierarchical}, provide alternatives to select $K$ in a principled way. However, it has been shown in a recent paper~\cite{NIPS2013_4880} that HDP is inconsistent for estimating the number of topics for LDA even with infinite amount of data.


In this paper, we
 provide theoretical analysis on the number of topics for latent topic models using spectral decomposition methods.
  By the results from Anandkumar et al.~\cite{Anima2012tensor},
  for the LDA model 
 the second-order moment follows a special structure as the summation over the outer product of topic vectors.
  We show that a spectral decomposition on the second-order empirical moment with proper thresholding on the singular values can lead to the correct number of topics. Under mild assumptions, we show that our analysis provides both 
  a
   lower bound and
   an
    upper bound on number of topics $K$ in the LDA model.
   To the best of our knowledge, this is the first work
    of analyzing the number of topics with provable guarantee
    by  utilizing the result of tensor decomposition approach. 

Our main contributions are:
\begin{itemize}
\item[(1)] For LDA, we analyze the empirical second-order moment and derive an upper bound on its variance in terms of the corpus statistics, i.e., the number of documents, the length of each document and the number of unique words. Essentially, our results provide a computable guideline to the convergence of second-{order} moment. This contribution itself is valuable,
  e.g., for determining the correct down-sampling rate on a large-scale dataset.
\item[(2)] We analyze the spectral structure of the true second-order moment for LDA. That is, we provide the spectral information on the covariance of \emph{Dirichlet} design matrix.
\item[(3)] Based on the results on empirical and true second-order moment for LDA, we derived three inequalities 
	regarding
 the number of topics $K$, which in turn provide both upper and lower bounds on $K$
  with known
   parameters or constants. We also present the simulation study for our theoretical results.
\item[(4)] We show that our results and techniques can be generalized to other mixture models.
 The results on \emph{Gaussian mixture models} is presented as an example.
\end{itemize}

The rest of the paper is organized as follows: In section~\ref{sec:lda}, we present our main result on how to analyze the number of topics in {the} LDA {model}. We carry out experiments on the synthetic datasets 
 to demonstrate the validity and tightness of our bounds in section~\ref{sec:exp}. We conclude the paper and
show how our methodology generalizes to other mixture models in section~\ref{sec:gen}.

\section{Analyze the Number of Topics in LDA}\label{sec:lda}

Latent Dirichlet Allocation~\cite{blei2003latent} (LDA)
 is a powerful generative model for topic modeling.
  It has been applied to a variety of  applications and also serves as building blocks {in}
other powerful models. Most existing methods follow the
empirical Bayes method for parameter estimation
~\cite{blei2003latent,teh2007collapsed, Griffiths2004,porteous2008fast}.
 Recently,
  method of moments  has been explored, leading to a series of interesting work and new insight into the LDA model.
   It has been shown in~\cite{anandkumar2012spectral,Anima2012tensor} that the {latent topics} can be directly derived from the properly constructed third-order moment (which can be directly estimated from the data)  by orthogonal tensor decomposition.
 Following this line of work, we observe that the low-order moments are also useful
 for discovering the number of topics in the LDA model.
   In this section, we will investigate
   the structure of both empirical and true second-order moment,
   and show that
   they lead to
    effective bounds on the number of topics.

\begin{table*}[t]
\centering
\caption{Notation for LDA }
\label{tbl:notation}
\begin{tabular}{ll}
\hline \hline
Notation 	&  		Definition	\\
\hline
$D$ ($d$)	& Number(index) of documents \\
$L$ ($\ell$)	& Number(index) of words in a document \\
$V$ ($v$)	& Number(index) of unique words  \\
$K$ ($k$)	& Number(index) of latent topics \\
$\bm{\mu}_k$	& Multinomial parameters for the $k$-th topic \\
$\bm{\mu}=\{\bm{\mu}_1,\ldots,\bm{\mu}_K\}$ & Collection of all topics\\
$\bm{w_d}=\{\mathbf{x}_{d\ell}\}_{\ell=1}^L $ & Collection of all words in $d$-th document \\
$\mathbf{x}_{d\ell}$& $\ell$-th word in $d$-th document\\
$\mathbf{h}_d$ & Topic mixing for $d$-th document \\
$z_{d\ell}$ & Topic assignment for word $\mathbf{x}_{d\ell}$ \\
$\bm{\alpha}=(\alpha_1,\ldots,\alpha_K)^{\top}$ & Hyperparameter for document topic distribution \\
$\bm{\beta}=(\beta_1,\ldots,\beta_V)^{\top}$ & Hyperparameter for generating topics\\
\hline \hline
\end{tabular}
\end{table*}

\subsection{Notation and Problem Formulation}
 As introduced in ~\cite{blei2003latent},
  the full generative process for the $d$-th document in the LDA model is described as follows:

\begin{enumerate}
\item Generate the topic mixing $\mathbf{h}_d \sim \text{Dir}(\bm{\alpha})$.
\item For each word $l = 1, \ldots, L$ in document $d$:
  \begin{enumerate}
    \item Generate a topic $z_{d\ell} \sim \text{Multi}(\mathbf{h}_d)$, where $\text{Multi}(\mathbf{h}_d)$ denotes the multinomial distribution.
    \item Generate a word $\mathbf{x}_{d\ell} \sim \text{Multi}(\bm{\mu}_{z_{d\ell}})$, where $\bm{\mu}_{z_{d\ell}}$ is the multinomial parameter associated with topic  $z_{d\ell}$.
  \end{enumerate}
\end{enumerate}
The notation is summarized in Table~\ref{tbl:notation}. $\mathbf{x}_{d\ell}$ is represented by natural basis $\mathbf{e}_v$,
meaning that the $\ell$-th word in $d$-th document is the $v$-th word in the dictionary.

In~\cite{Anima2012tensor},  the authors proposed the method of moment for learning the LDA model, where the empirical first-order moment $\hat{\Moment}_1$ is defined as
$$
\hat{\Moment}_1 =  \frac{\sum_d \sum_{\ell} \mathbf{x}_{d\ell}}{DL},
$$
and the empirical second-order moment $\SecMH$ as
$$
\SecMH = \frac{\sum_d \sum_{\ell \neq \ell^\prime} \mathbf{x}_{d\ell }\otimes \mathbf{x}_{d\ell^\prime}}{DL(L-1)} - \frac{\alpha_0}{\alpha_0 + 1} \hat{\Moment}_1 \otimes \hat{\Moment}_1,
$$
where $\alpha_0=\sum_{k=1}^K\alpha_k$ and
  the outer product is defined as $\mathbf{x} \otimes \mathbf{x} :=\mathbf{x} \mathbf{x}^{\top}$ for any column vector $\mathbf{x}$.
   Then we define the first-order and second-order moments as the expectation of the empirical moments, i.e., $\Moment_1=\Expect{\hat{\Moment}_1}$ and $\SecM = \Expect{\SecMH}$ respectively. Furthermore, it has been shown that $\SecM$ equals the weighted sum of the outer products of the topic parameter $\bm{\mu}$~\cite{Anima2012tensor}, i.e.,
$$\SecM = \sum_{k=1}^K \frac{\alpha_k}{(\alpha_0+1)\alpha_0}	  \bm{\mu}_k \otimes \bm{\mu}_k.$$

This implies that the rank of $\SecM$ is exactly the number of topics $K$.
 Another interesting observation from this derivation is that since $\SecM$ is the summation of $K$ rank-1 matrices and all the topics $\bm{\mu}_k$ are linearly independent almost surely under our full generative model, we have the K-th largest singular value $\SVK{K}{\SecM}>0$ and K+1-th largest singular value $\SVK{K+1}{\SecM}=0$.
  Therefore, the number of non-zero singular values of $\SecM$ is exactly the number of topics,
   which provides a direct way to estimate $K$ under the noiseless scenario.
   However, in practice,
    we only have access to the estimated $\SecMH$ as an approximation to the true second-order moment $\SecM$.
    As a result, the rank of $\SecMH$ may not be $K$ and $\SVK{K+1}{\SecMH}$ may be larger than zero.
    To overcome this obstacle, we need to study
     (1) the spectral structure of $\SecM$, and
     (2) the relationship between $\SecM$ and its estimator $\SecMH$.

\subsection{Solution Outline}
The second-order moment $\mathbf{M}_2$ can be estimated directly from the observations,
 without inferring the topic mixing and estimating  parameters.
  Our idea follows that when the sample size becomes large enough, $\SecMH$ can approximate $\SecM$ well enough. That is, $\SVK{K+1}{\SecMH}$ is very close to zero while $\SVK{K}{\SecMH}$ is bounded away from zero.
  Then, by picking a proper threshold $\theta$ satisfying $\SVK{K+1}{\SecMH}<\theta<\SVK{K}{\SecMH}$,
    we can obtain the value of $K$ by simply counting the number of singular values of $\SecMH$ greater than $\theta$.
     We will work along two directions to achieve the goal: (1) examine the convergence rate of the singular values of $\SecMH$; (2) investigate the relationship between the spectral structure of $\SecM$ and the model parameters. Next we will provide the analysis results from both directions.

\subsection{Convergence of  $\SecMH$}

Without loss of generality,
 we assume that both $\bm{h}_k$ and $\bm{\mu}_k$ are generated from symmetrical Dirichlet distribution,
  namely $\alpha_k=\alpha$ for $k=1, \ldots, K$ and $\beta_v=\beta$ for $v=1, \ldots, V$.
   We also assume that all documents have the same length $L$ for simplicity.
    Since $\SecMH$ is an unbiased estimator of $\SecM$ by definition,
     we can bound the difference between the singular value of $\SecMH$ and that of $\SecM$
      by bounding their variance as follows:

\begin{theorem}\label{thm:lda:noise}
For the LDA model, with probability at least $1-\delta$, we have
\begin{equation}\label{eq:lda:cu}
|\SVK{i}{\SecMH} - \SVK{i}{\SecM}| \leq \deltaR , 1 \leq i \leq V\ \nonumber
\end{equation}
where
$
\deltaR = \frac{1}{\sqrt{D\delta}}\sqrt{\frac{2}{L^2}+\frac{2}{V^2}+\mathcal{O}(\epsilon)}
$, $\epsilon$ represents higher-order terms.

Especially, when $i\geq K+1$, we have
\begin{equation}\label{eq:lda:cup}
\SVK{i}{\SecMH}
\leq \deltaR.
\end{equation}
\end{theorem}
\begin{proof}
Let $\RM=\SecM-\SecMH$ and $\TNorm{\RM}, \FNorm{\RM}$ be the spectral and Frobenius norm of $\RM$, respectively. We denote $\EVK{i}{\Moment}$ as the $i$-th largest eigenvalue of matrix $\Moment$.
We establish the result through the following chain of inequalities:
\begin{align*}
\max_i |\SVK{i}{\SecMH} - \SVK{i}{\SecM}|
\stackrel{(i)}{\leq}& \max_i |\EVK{i}{\SecMH} - \EVK{i}{\SecM}| \\
 \stackrel{(ii)}{\leq}&  \TNorm{\RM} \\
 \stackrel{(iii)}{\leq}&  \FNorm{\RM}.
\end{align*}

Step (i) follows directly based on the fact that $\mathbf{M}_2$ is semi-definiteness and $\hat{\mathbf{M}}_2$ is symmetric.
 {The} detailed proof is deferred to Lemma \ref{lem:eigsvd} in Appendix.
  Step (ii) and (iii) are well-known results on matrix norm and matrix perturbation theory~\cite{hornmatrix}.
   And in Lemma~\ref{lem:lda:ub}, we provide upper bound on the Frobenius norm of matrix $\RM$.
    Because $\text{Rank}(\mathbf{M}_2)\leq K$, i.e., $\sigma_i(\mathbf{M}_2)=0$ for $i\geq K+1$,
    therefore, $\SVK{i}{\SecMH}\leq \deltaR$.
\end{proof}

\begin{lemma}\label{lem:lda:ub}
For the LDA model, with probability at least $1-\delta$, we have $\FNorm{\RM}\leq\deltaR$.
\end{lemma}
\begin{proof}
We first compute the expectation $\Expect{\FNorm{\RM}^2}$ and then use Markov inequality to complete the proof. The square of Frobenius norm is $
\FNorm{\RM}^2 = \sum_{i,j} \RMij{i}{j}^2
$. Since we have $\ExpectC{\RMij{i}{j}}{\bm{\mu}}=0$, $Var[\RMij{i}{j}|\bm{\mu}]=\ExpectC{\RMij{i}{j}^2}{\bm{\mu}}-\ExpectTC{\RMij{i}{j}}{\bm{\mu}}=\ExpectC{\RMij{i}{j}^2}{\bm{\mu}}$.
The expectation of $\FNorm{\RM}^2$ can be calculated as
\begin{align*}
\Expect{\FNorm{\RM}^2} = & \Expect{\ExpectC{\FNorm{\RM}^2}{\bm{\mu}}} \\
 =& \Expect{\sum_{i\neq j}Var[\RMij{i}{j}|\bm{\mu}] + \sum_{i}Var[\RMij{i}{i}|\bm{\mu}]}.
\end{align*}

The remaining task is to calculate the conditional variance of $\RMij{i}{j}$ and $\RMij{i}{i}$, which is discussed in Lemma~\ref{lem:lda:var}.

Then by Markov inequality, for any $t>0$, we have
$$
\text{Pr}(||\mathbf{R}||_\text{F}^2 \geq t\times \mathbb{E}[||\mathbf{R}||_\text{F}^2] ) \leq 1/t
$$

By setting $t=1/\delta$, with probability at least $1-\delta$, we have
$$
||\mathbf{R}||_\text{F} \leq \frac{1}{\sqrt{D\delta}}\sqrt{\frac{2}{L^2}+\frac{2}{V^2}+\mathcal{O}(\epsilon)} = \deltaR.
$$
\end{proof}
\begin{lemma}\label{lem:lda:var}
For the LDA model, the following holds
\begin{align*}
\mathbb{E}[Var[\mathbf{R}_{ij}|\bm{\mu}]] \leq \frac{1}{DL^2V^2}+\frac{2}{DV^4}+\mathcal{O}(\epsilon),\quad \forall i\neq j,
\end{align*}
and
\begin{align*}
\mathbb{E}[Var[\mathbf{R}_{ii}|\bm{\mu}]] \leq \frac{1}{DL^2V}+\frac{2}{DV^4}+\mathcal{O}(\epsilon),\quad \forall i,
\end{align*}
for $i,j=1,2,\dots,V$ and $\epsilon$ represents higher-order terms.
\end{lemma}

We make a few relaxations and introduce $\mathcal{O}(\cdot)$ notation (keeping the dominant terms and absorb the rest into $\mathcal{O}(\epsilon)$ to achieve an upper-bound on the variance). To be rigorous, we have the following assumptions on the scale of each statistics or parameters: $L=\mathcal{O}(D)$, $V=\mathcal{O}(D)$, $L=\mathcal{O}(V)$, $K=\mathcal{O}(L)$, $K={\Omega}(1)$, $\alpha=\Theta(1)$, and $\beta=\Theta(1)$. The calculation of the variance is provided in Appendix \ref{apd:lda}.

It is interesting to examine the role of $D,L$, and $V$ in $\deltaR$. $\deltaR$ decreases to $0$ as $D \to +\infty$. Even if there are only two words in each document, $\SecMH$ would still converge to $\SecM$.
 Similar observation is made in \cite{Anima2012tensor}.
  $L$ and $V$ have similar influence on $\deltaR$.

   To apply the results above, we simply ignore the higher-order terms.
    However, because $\epsilon$ will increase as $\alpha$, $\beta$, or $K$ decreases,
     one should pay extra attention when $D,L,V$ are far from the asymptotic region.
     As shown in our simulated studies, our bound yields convincing results
      when $D,L,V$ are on the scale of hundreds or above, which is more than common in real-world applications.

\subsection{Spectral Structure of $\mathbf{M}_2$}\label{sec:spstr}

The spectral structure of $\mathbf{M}_2$  depends on $K,V$ and $\bm{\mu}_k,\alpha_k,k=1,2,\dots,K$. We use the following theorem to characterize the spectral structure of $\mathbf{M}_2$.

\begin{theorem}\label{thm:mul}
Assume that $\alpha_{\min} = \min_k \{\alpha_k\}$, $\alpha_{\max} = \max_k \{\alpha_k\}$, and $\beta_v=\beta,\forall v=1,\dots,V$ and
\begin{align*}
  \delta'=& \left( \frac{\log(K/\delta_3)K\pq{\beta+2\log\pq{K/\delta_2}}^2}{V\beta}  \right)^{\frac{1}{2}},
\end{align*}
\begin{itemize}
\item[(1)] With probability at least
$
1-\delta_1-\delta_2-\delta_3,
$
 we have
\begin{align*}
&\sigma_1(\mathbf{M}_2)
\leq
\overline{\sigma_1} \numberthis \label{eqn:Cl1} \\
= & \frac{\alpha_{\max}}{\alpha_0(\alpha_0+1)}\frac{(1+\delta')V(\beta+K\beta^2)}{\max\left\{0_+,V\beta-\sqrt{2V\beta\log(K/\delta_1)}\right\}^{2}}.
\end{align*}
\item[(2)] With probability at least $1-\delta_1-\delta_2-\delta_3$, we have
\begin{align*}
& \sigma_K(\mathbf{M}_2)
\geq
\underline{\sigma_K} \numberthis \label{eqn:ClK} \\
=&\frac{\alpha_{\min}}{\alpha_0(\alpha_0+1)} \frac{(1-\delta')V\beta}{(V\beta+2\sqrt{V\beta}\log(K/\delta_1))^{2}}.
\end{align*}
\end{itemize}
\end{theorem}

\begin{proof}

We have $\mathbf{M}_2=\frac{1}{\alpha_0(\alpha_0+1)}\sum_{k=1}^{K}\alpha_k \bm{\mu}_k \otimes \bm{\mu}_k =\frac{1}{\alpha_0(\alpha_0+1)}\mathbf{O} \mathbf{A} \mathbf{O}^{\top}$,
 where $\mathbf{O}=(\bm{\mu}_1,\dots,\bm{\mu}_K)$ is a $V\times K$ matrix
  and $\mathbf{A}=\text{diag}(\alpha_1,\dots,\alpha_K)$ is a diagonal matrix. The  first $K$ singular values of $\mathbf{M}_2$ are also the first $K$ singular values of $\frac{1}{\alpha_0(\alpha_0+1)} \mathbf{A}^{\frac{1}{2}}\mathbf{O}^{\top}\mathbf{O}\mathbf{A}^{\frac{1}{2}}$. And we have
\begin{align*}
\sigma_1(\mathbf{A}^{\frac{1}{2}}\mathbf{O}^{\top}\mathbf{O}\mathbf{A}^{\frac{1}{2}}) \leq \sigma_1(\mathbf{A})\sigma_1(\mathbf{O}^{\top}\mathbf{O}),
\end{align*}
and
\begin{align*}
\sigma_K(\mathbf{A}^{\frac{1}{2}}\mathbf{O}^{\top}\mathbf{O}\mathbf{A}^{\frac{1}{2}}) \geq \sigma_K(\mathbf{A})\sigma_K(\mathbf{O}^{\top}\mathbf{O}).
\end{align*}

To estimate the singular value of $\mathbf{O}^{\top}\mathbf{O}$,
 we need to utilize the fact that $\bm{\mu}_k \sim \text{Dir}(\beta)$.
  The random variables in the same column of $\mathbf{O}$ are dependent with each other. Thus, powerful results from random matrix theory can not be applied.
 To decouple the dependency, we design a diagonal matrix $\mathbf{\Lambda}$,
  whose diagonal elements are drawn from $\text{Gamma}(V\beta ,1)$ independently.
   In this way, $\hat{\mathbf{O}}=\mathbf{O}\mathbf{\Lambda}$ is a matrix with independent elements,
    i.e., each element is an i.i.d. random variable following $\text{Gamma}(\beta,1)$.

We denote each row of $\hat{\mathbf{O}}$ as $\mathbf{r}_v,v=1,\dots,V$, then $\hat{\mathbf{O}}^{\top}\hat{\mathbf{O}}=\sum_{v=1}^{V}\mathbf{r}_v^{\top}\mathbf{r}_v$. In order to apply matrix Chernoff bound~\cite{Tropp2012matrix},
 we need to bound the spectral norm of $\mathbf{r}_v^{\top}\mathbf{r}_v$, i.e., $\max_v \{ \sigma_{1}(\mathbf{r}_v^{\top}\mathbf{r}_v)\}$.
  Because $\mathbf{r}_v^{\top}\mathbf{r}_v$ is a rank-$1$ matrix, we have $\sigma_{1}(\mathbf{r}_v^{\top}\mathbf{r}_v)=\mathbf{r}_v \mathbf{r}_v^{\top}$. By Lemma \ref{crl:SSGRV} (see Appendix) and the union bound, with probability greater than $1-KVe^{-\frac{c_1}{2}\min\{\frac{c_1}{2},\sqrt{\beta}\}}$, we have
$$
R=\max_{v=1,\ldots,V}\{ \sigma_{1}(\mathbf{r}_v^{\top}\mathbf{r}_v)\} \leq K(\beta+c_{1}\beta^{1/2})^2.
$$

We also have $\sigma_{1}(\mathbb{E}[\hat{\mathbf{O}}^{\top}\hat{\mathbf{O}}])=V\beta(1+K\beta)$ and $\sigma_{K}(\mathbb{E}[\hat{\mathbf{O}}^{\top}\hat{\mathbf{O}}])=V\beta$. Applying the matrix Chernoff bound to $\hat{\mathbf{O}}^{\top}\hat{\mathbf{O}}$, with probability greater than
\begin{align*}
1-KVe^{-\frac{c_1}{2}\min\{\frac{c_1}{2},\sqrt{\beta}\}}-K\left[\frac{e^{-\delta'}}{(1-\delta')^{1-\delta'}}\right]^{\frac{V\beta}{K(\beta+c_{1}\beta^{1/2})^2}},
\end{align*}
we have
\begin{align*}
\SVK{K}{\hat{\mathbf{O}}^{\top}\hat{\mathbf{O}}}\geq (1-\delta')V\beta.
\end{align*}

And with probability greater than
\begin{align*}
1-KVe^{-\frac{c_1}{2}\min\{\frac{c_1}{2},\sqrt{\beta}\}}-K\left[\frac{e^{\delta'}}{(1+\delta'})^{1+\delta'}\right]^{\frac{V\beta}{K(\beta+c_{1}\beta^{1/2})^2}},
\end{align*}
we have
\begin{align*}
\SVK{1}{\hat{\mathbf{O}}^{\top}\hat{\mathbf{O}}}\leq (1+\delta')V\beta(1+K\beta).
\end{align*}

By definition, for $i=1,\dots,K$, it follows
\begin{align*}
\sigma_i(\mathbf{M}_2)=\frac{1}{\alpha_0(\alpha_0+1)} \sigma_i(\mathbf{A}^{\frac{1}{2}}\mathbf{\Lambda}^{-1}\hat{\mathbf{O}}^{\top}\hat{\mathbf{O}}\mathbf{\Lambda}^{-1}\mathbf{A}^{\frac{1}{2}}).
\end{align*}
Therefore, we have
\begin{align*}
\sigma_{1}(\mathbf{M}_2) \leq \frac{\alpha_{\max}}{\alpha_0(\alpha_0+1)} \frac{\SVK{1}{\hat{\mathbf{O}}^{\top}\hat{\mathbf{O}}}}{\sigma^2_{K}(\mathbf{\Lambda})},
\end{align*}
and
\begin{align*}
\sigma_{K}(\mathbf{M}_2) \geq \frac{\alpha_{\min}}{\alpha_0(\alpha_0+1)} \frac{\SVK{1}{\hat{\mathbf{O}}^{\top}\hat{\mathbf{O}}}}{\sigma^2_{1}(\mathbf{\Lambda})}.
\end{align*}
Since $\SVK{1}{\mathbf{\Lambda}}$ and $\SVK{K}{\mathbf{\Lambda}}$ are the maximum and minimum of a set of random variables following $\text{Gamma}(V\beta,1)$, we can bound them by Lemma \ref{lem:GammaMaxMin} with coefficient $c_2$. Proper choices of coefficients $c_1,c_2,\delta'$ (provided in Appendix \ref{apd:mul}) leads to the conclusions of Theorem \ref{thm:mul}.
\end{proof}

With certain assumptions on $\alpha_{\max}$ and $\alpha_{\min}$, we can fully utilize the bounds above.
 If we assume that $\alpha_{k}=\Theta(\frac{1}{K}\sum_i \alpha_i)=\Theta(1)$, $\forall k$,
  then $\frac{\alpha_{\min}}{\alpha_0}=\Theta(\frac{1}{K})$ and $\alpha_0=\Theta(K)$.
   Therefore, $\underline{\sigma_K}$ decreases rapidly as $K$ increases, where $\sigma_K(\mathbf{M}_2) \propto \frac{1}{K^2}$ approximately. {This} fact leads to increasing difficulty in distinguishing the topics with small singular values from noise. Note that $\overline{\sigma_1}$ also decreases with a slower rate  as $K$ increases.

\subsection{Analysis of the Number of Topics}\label{sec:lda:sum}

\begin{figure*}[!ht]
\begin{center}
\begin{tabular}{ccc}
\includegraphics[scale=0.35]{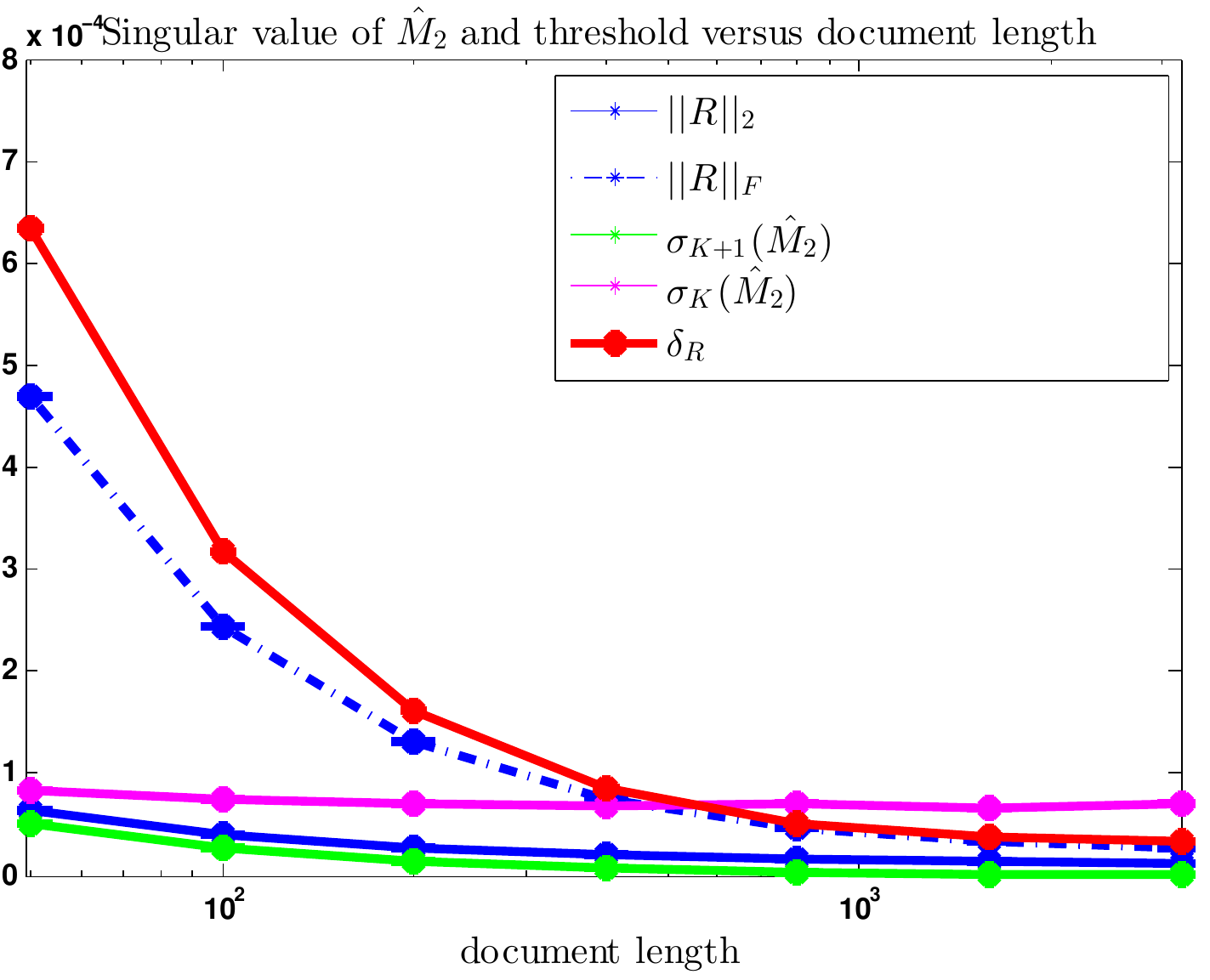}&
\includegraphics[scale=0.35]{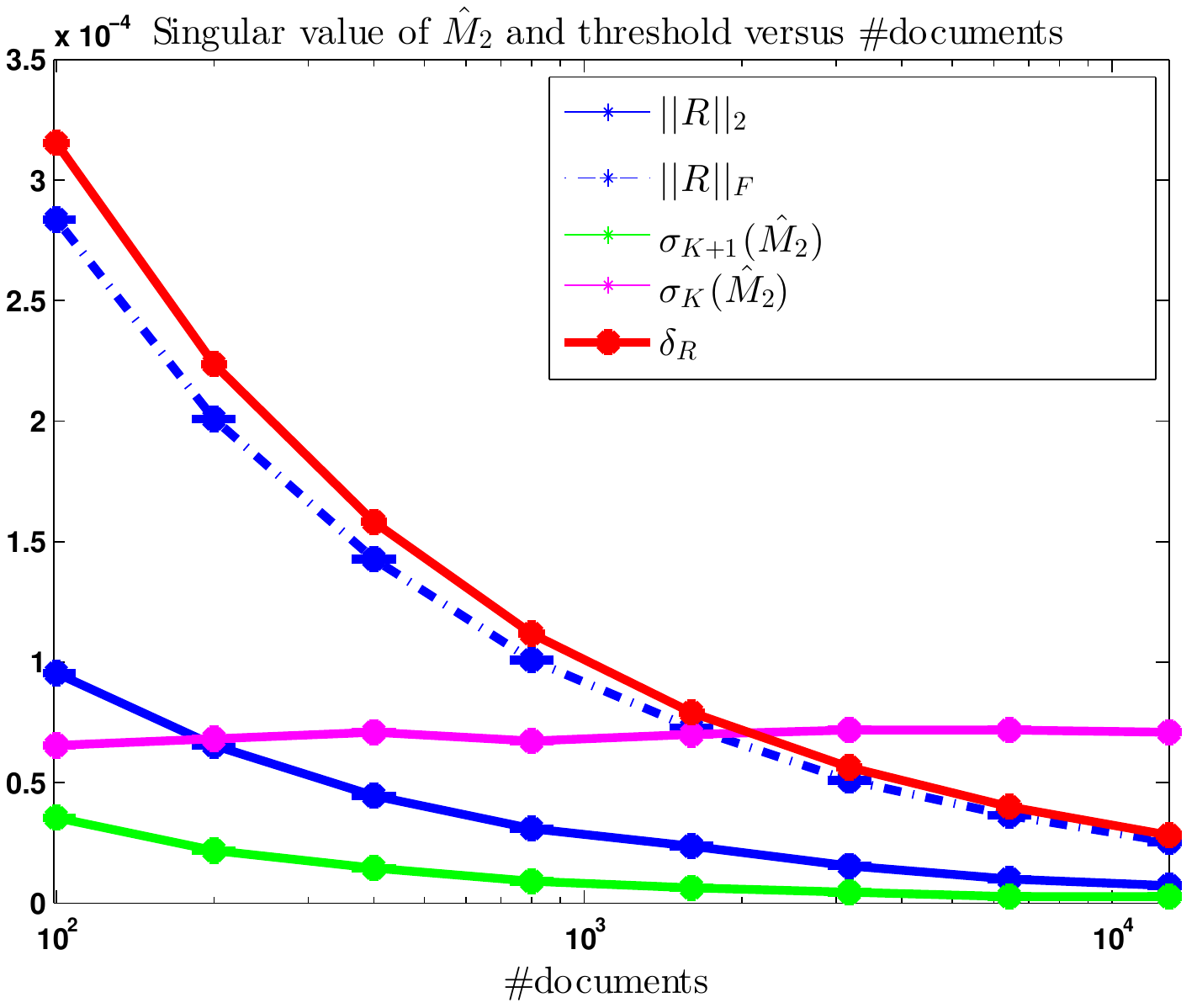}&
\includegraphics[scale=0.35]{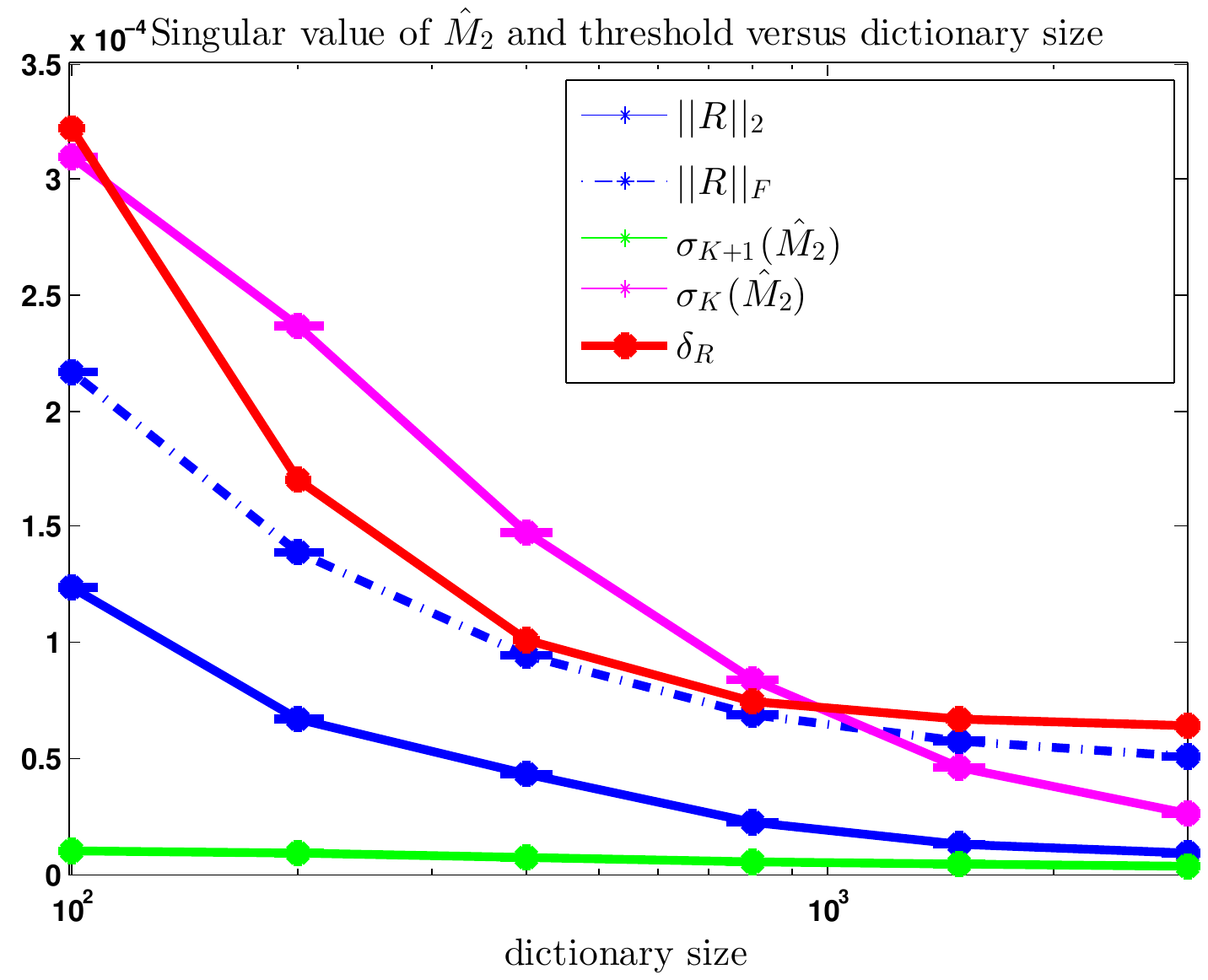}\\
(a) & (b) & (c)\\
\includegraphics[scale=0.35]{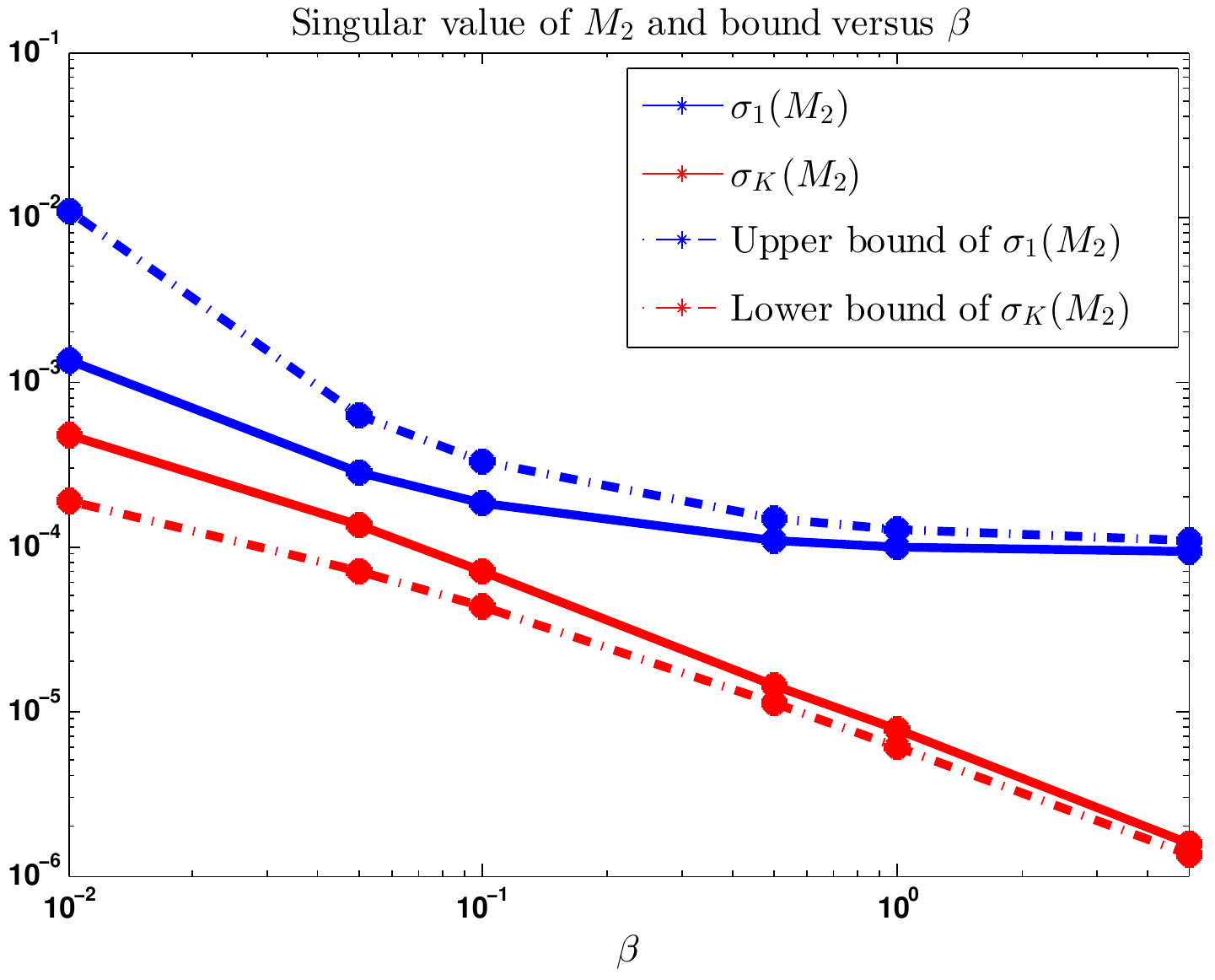}&
\includegraphics[scale=0.35]{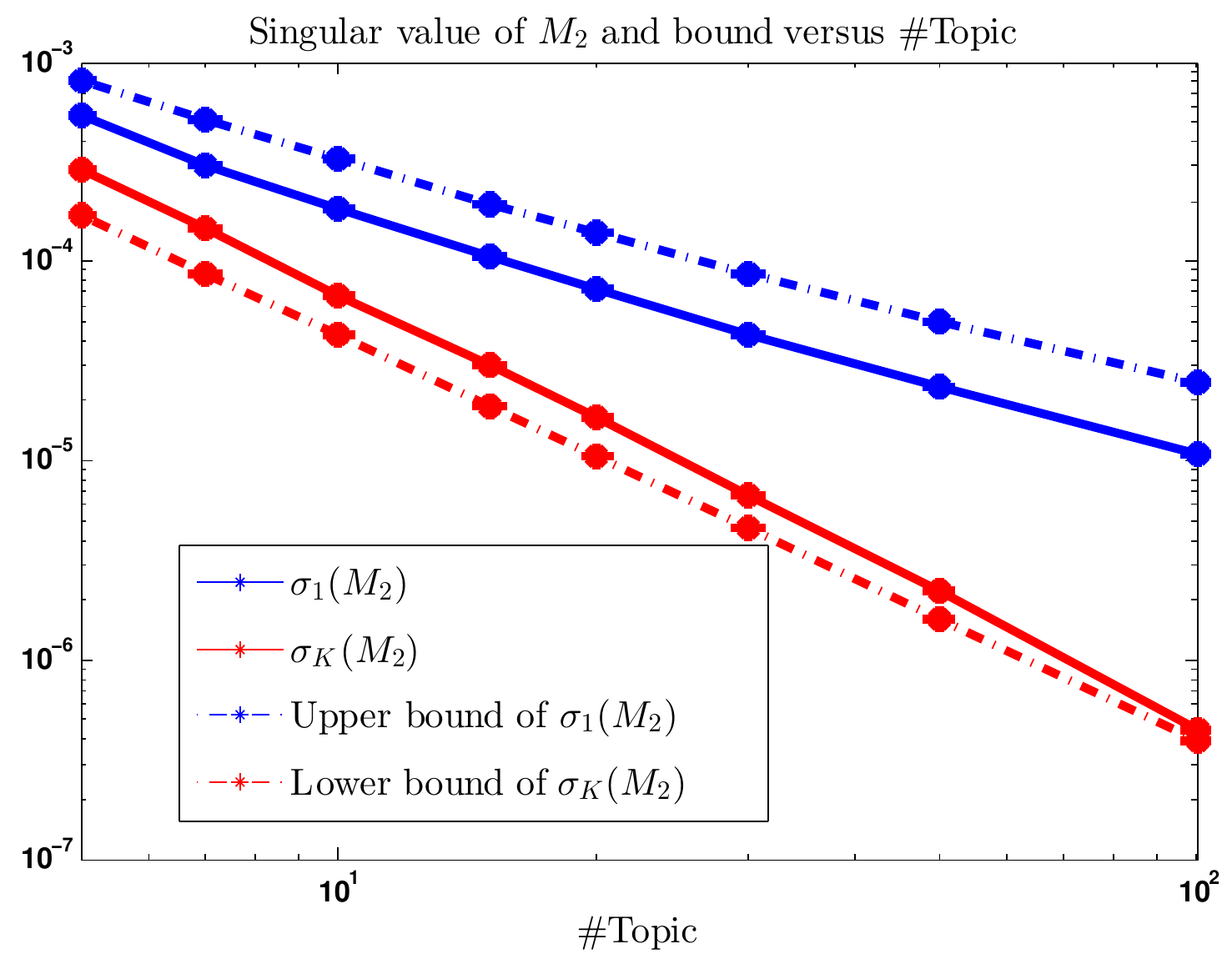}&
\includegraphics[scale=0.35]{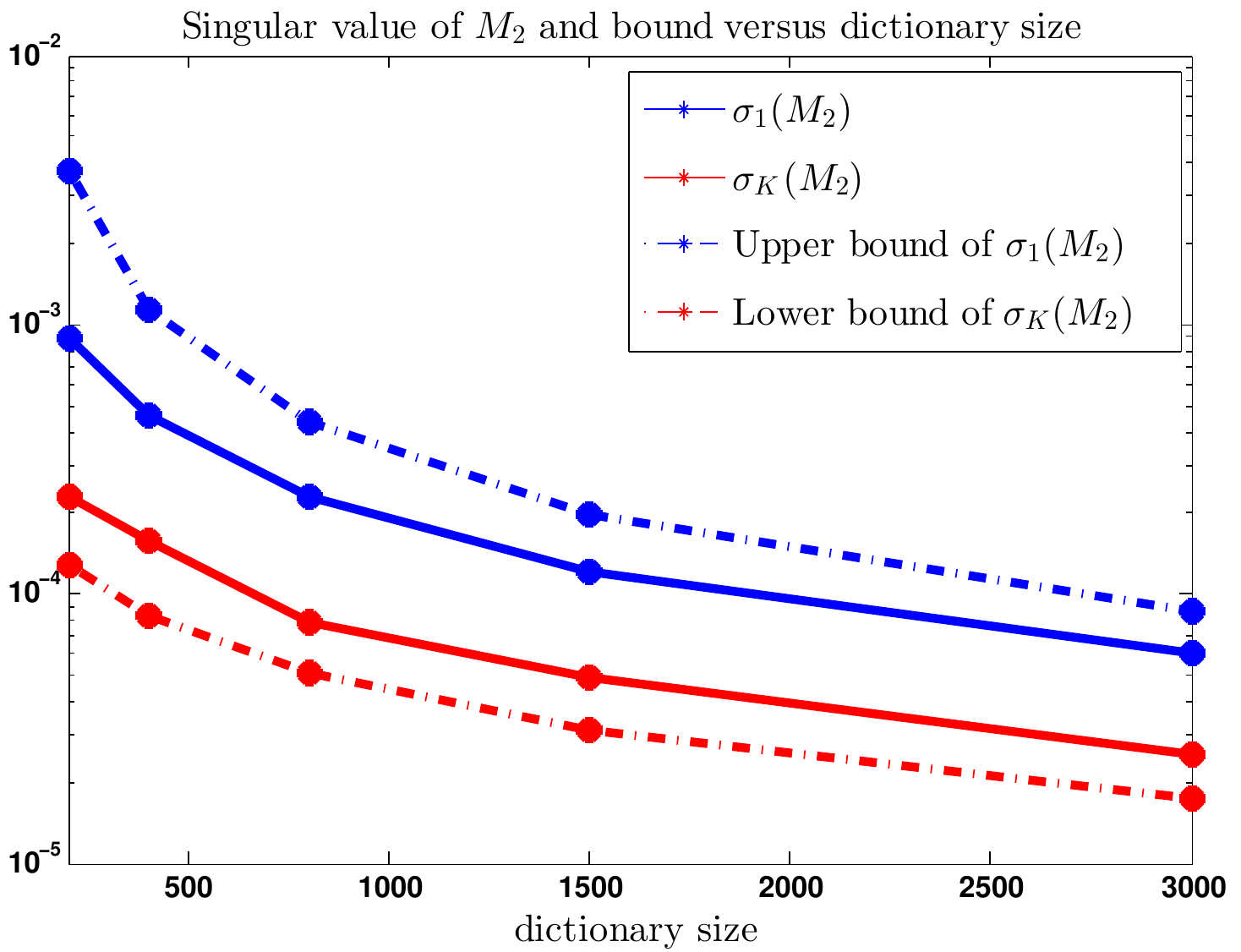}\\
(d) & (e) & (f)\\
\end{tabular}
\end{center}
\caption{Experimental results on synthetic data under LDA model. Results on $\deltaR$ are illustrated in Figure (a-c). $\underline{\sigma_K}$ and $\overline{\sigma_1}$ are illustrated in Figure (d-f).}\label{fig:lda}
\end{figure*}

The convergence of $\hat{\mathbf{M}}_2$ and the spectral structure of $\mathbf{M}_2$ provide us the upper bounds and the lower bounds on the singular values of the empirical second-order moments $\hat{\mathbf{M}}_2$.
 We can infer the number of topics by
 the following steps:

First, by setting $\theta > \deltaR$, thresholding provides a lower bound on $K$, since with high probability, every spurious topic has singular value smaller than $\deltaR$.\footnote{Strictly speaking, there is no one-to-one correspondence between topics and the singular values of the second-order moments. Here we refer to the correspondence in terms of the total number of topics.}

Secondly, if we set $\theta <  \underline{\sigma_K} - \deltaR$,
thresholding provides a upper bound on $K$, since with high probability, every true topic has singular value greater than the threshold. However, the above threshold is not computable,
since $\underline{\sigma_K}$ depends on the true number of topics $K$.

Instead, we can directly utilize the upper bound $\overline{\sigma_1}$ on $\sigma_1(\hat{\mathbf{M}}_2)$ to provide an upper bound for $K$.
 We have $\sigma_1(\hat{\mathbf{M}}_2) \leq \overline{\sigma_1}+ \delta_{\mathbf{R}} \label{eq:lda:upperK}$ as shown in Theorem~\ref{thm:mul}.
  The left hand side, $\sigma_1(\hat{\mathbf{M}}_2)$, is determined by the observed corpus, and the right hand side $\overline{\sigma_1}+ \delta_{\mathbf{R}}$ is a function of $K$.
   When $\overline{\sigma_1}+ \delta_{\mathbf{R}}$ decreases as $K$ increases (see discussion in Section~\ref{sec:spstr}), solving the inequality leads to an upper bound on $K$.

\section{Experimental Results}\label{sec:exp}
We validate our theoretical results by conducting experiments on the synthetic datasets generated according to the LDA model. For each experiment setting, we report the results by averaging over five random runs.

In the first set of experiments, we test the convergence 
{of the} second-order moment $\mathbf{\hat{M}}_2$ {as a function of $\deltaR$}. The parameter setting is as follows: $K=10$, $\forall k, \alpha_k=1$ and $\forall v, \beta_v=0.1$. We vary the dictionary size $V$, document length $L$, or document number $D$ while keeping the other two fixed. The detailed {settings are} summarized as belows:
\begin{enumerate}[(a)]
\item Fix $D=2000$ and $V=1000$, vary the length of document $L$ from $50$ to $3200$.
\item Fix $L=500$ and $V=1000$, vary  the number of documents $D$ from $100$ to $12800$.
\item Fix $L=500$ and $D=2000$, vary  the size of dictionary $V$ from $100$ to $3000$.
\end{enumerate}

Figure~\ref{fig:lda} (a-c) shows the matrix norms on $\RM = \mathbf{\hat{M}}_2-\mathbf{{M}}_2$ and the $K$-th and $(K+1)$-th largest singular values of $\mathbf{\hat{M}}_2$. The results match nicely with our theoretical analysis in that $\deltaR$ serves as an accurate upper bound on the Frobenius norm of $\hat{\mathbf{M}}_2-\mathbf{M}_2$. When the amount of data is {large} enough, the red line goes below the purple line, which indicates that with enough data, thresholding with $\deltaR$ provides a tight lower bound on the number of topics.

In the second experiment, we evaluate our bounds on the spectral structure of $\mathbf{M}_2$ in Theorem~\ref{thm:mul}. Similarly, we vary $K, \beta$, or $V$ while keeping the other two parameters fixed. The detailed {settings are} as follows:
\begin{enumerate}[(a)]
\setcounter{enumi}{3}
\item Fix $\alpha_k=1$, $V=1000$, and $K=10$, 		 vary $\beta_v=\beta$ from $0.01$ to $5$.
\item Fix $\alpha_k=1$, $V=1000$, and $\beta_v=0.1$, vary number of topics $K$ from $5$ to $100$.
\item Fix $\alpha_k=1$, $K=10$  , and $\beta_v=0.1$, vary the size of dictionary $V$ from $200$ to $3000$.
\end{enumerate}
The results in Figure \ref{fig:lda} (d-f) match well with our theoretical analysis.

In the last experiment, we calculate the upper bound and the lower bound of $K$ when varying the number of documents {or the length of documents}. The results are presented in Figure \ref{fig:lda:k}. As we can see, the lower bound indeed converges to the true number of topics. However, the upper bound converges to a value other than the ground truth, partly because the upper bound involves both $\overline{\sigma_1}$ and $\deltaR$, whereas $\overline{\sigma_1}$ does not change as the size of dataset increases. The experiment results demonstrate that our upper and lower bounds on $K$ can effectively narrow down the range of possible $K$.
\begin{figure*}[ht]
\begin{center}
\begin{tabular}{ccc}
\includegraphics[scale=0.35]{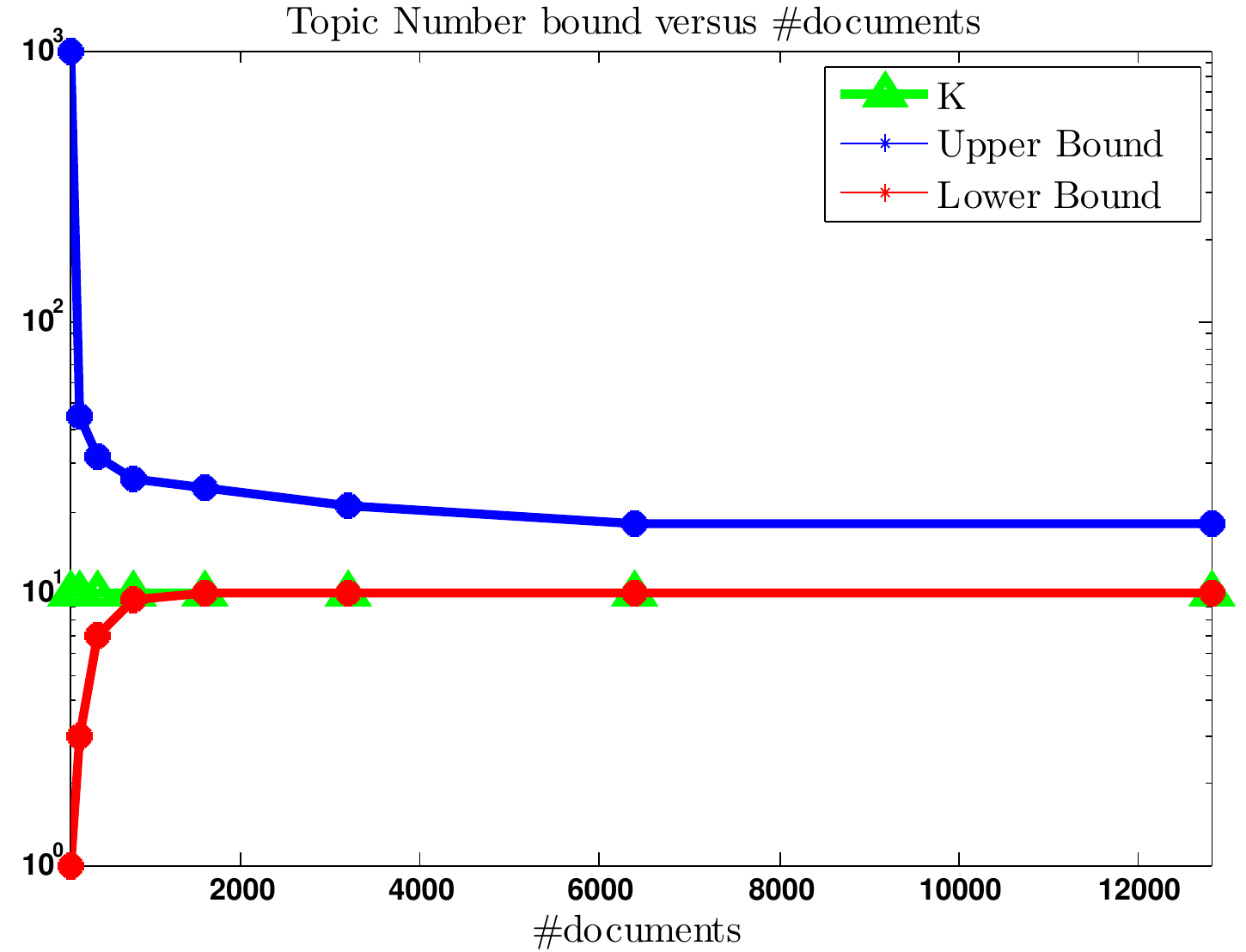}&
\includegraphics[scale=0.35]{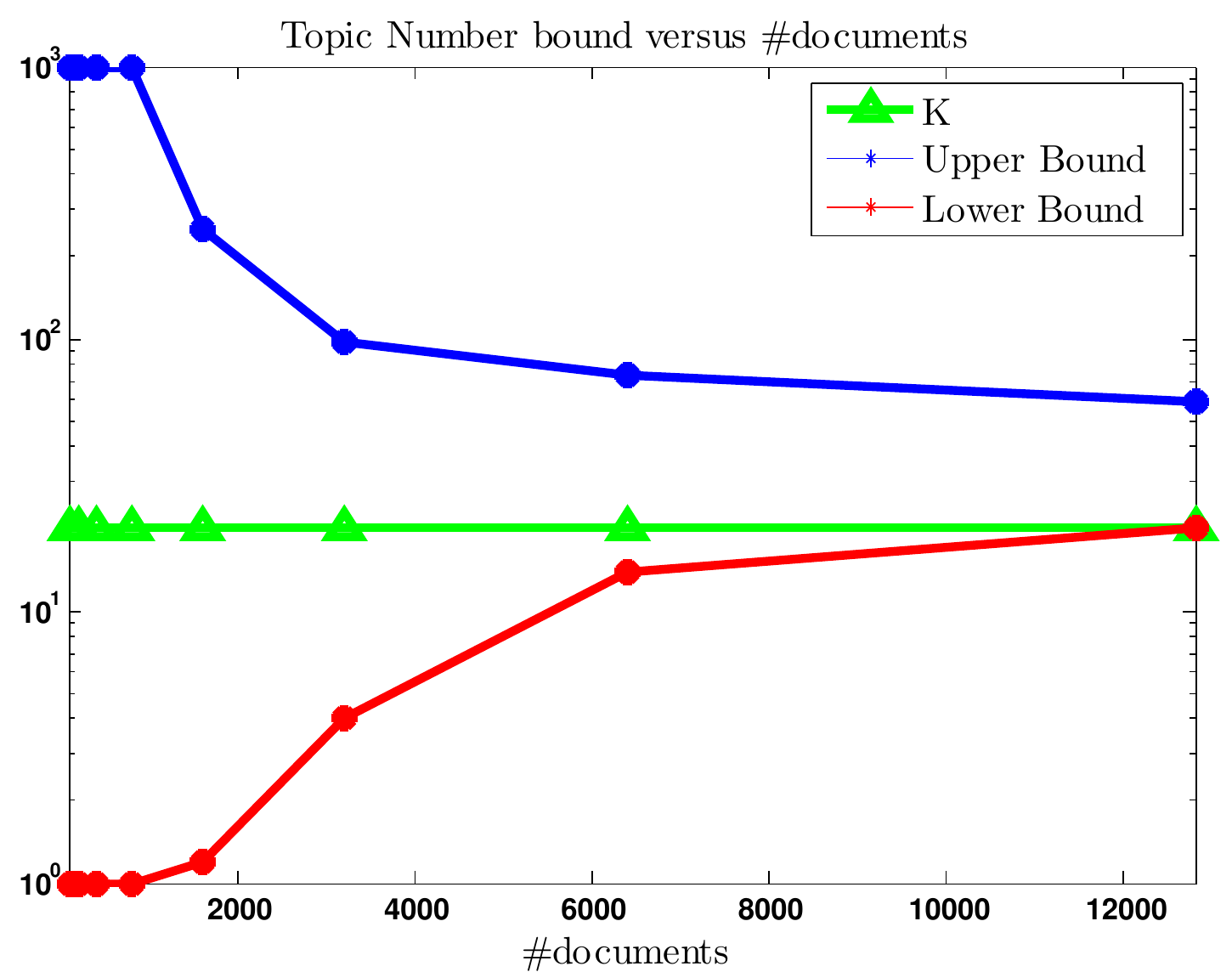}&
\includegraphics[scale=0.35]{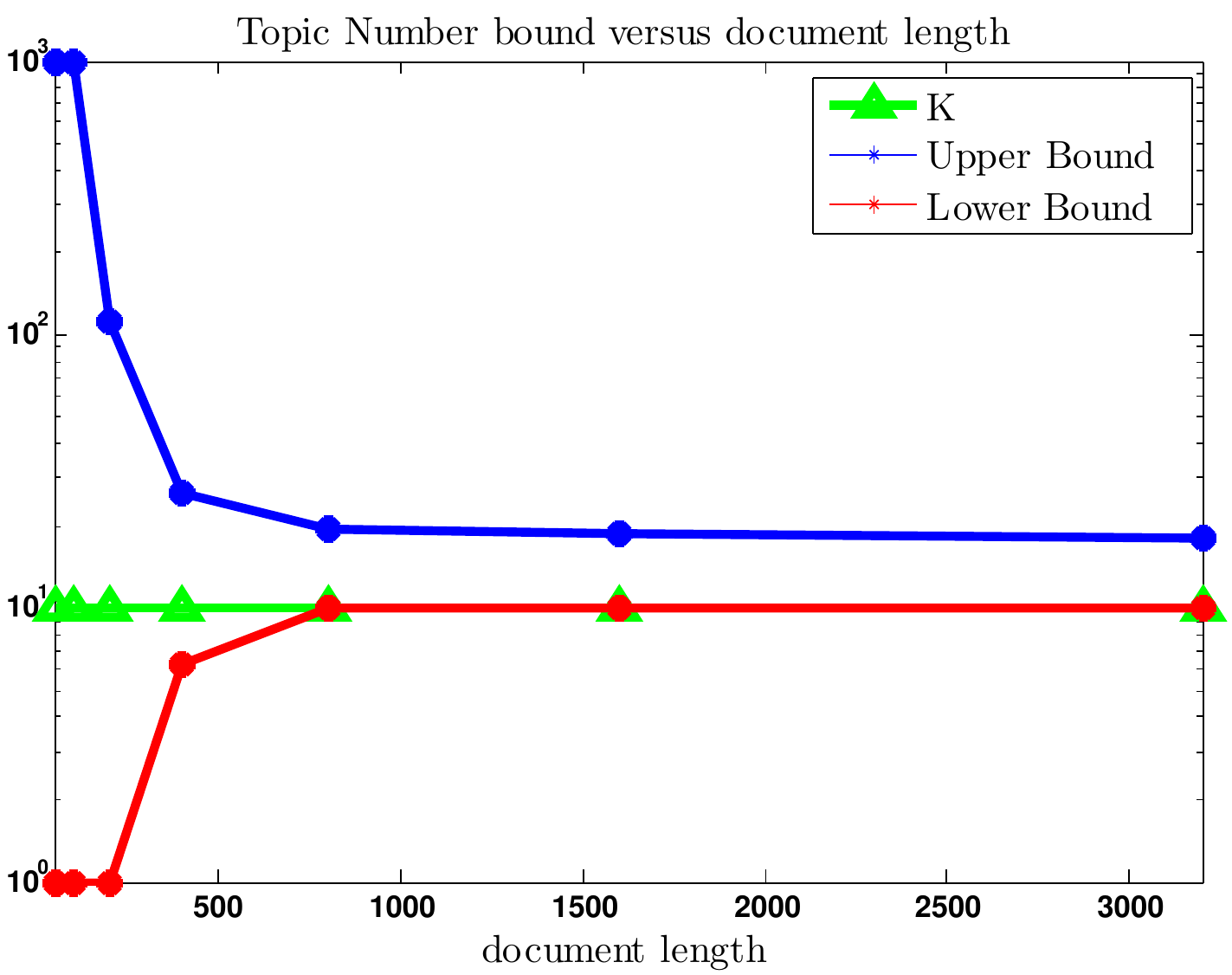}\\
(a)$K=10,L=500$, varying $D$ & (b)$K=20,L=500$, varying $D$  & (c)$K=10,D=2000$, varying $L$\\
\end{tabular}
\end{center}
\caption{The upper and lower bounds on number of topics for LDA based on discussion in section \ref{sec:lda:sum}.}\label{fig:lda:k}
\end{figure*}

\section{Discussion and Conclusions}\label{sec:gen}
So far we have shown that for the LDA model, by investigating the convergence  of the empirical moments $\hat{\mathbf{M}}_2$ and the spectral structure of the expected moment $\mathbf{M}_2$, the singular value of the empirical moment provides useful information on the number of topics. This line of research provides an interesting direction for analyzing mixture models in general \cite{hsu2013learning}. Next we show how to generalize our methodology with an example of Gaussian Mixture Models (GMM).

\subsection{Generalization}
Our analysis can be easily generalized to other mixture models whose empirical low-order moments have the same structures as the weighted sum of the outer products of mixture components. Convergence analysis of $\deltaR$ leads to the lower bound on the number of mixture components, while solving inequality on the first singular value $\sigma_1(\hat{\mathbf{M}}_2)$ provides an upper bound. In order to derive the convergence bound $\delta_{\mathbf{R}}$, the variance of $\mathbf{R}_{ij}$ need to be computed. Moreover, we need to explore the spectral structure of the true moment to provide upper and lower bound on the first and the $K$-th singular values respectively.
 
 As an example, we  next show how to conduct the analysis on the \emph{Gaussian Mixture Model}~\cite{Bishop2006PRML} with spherical mixture components.

GMM assumes that the data points are generated from a mixture of multivariate Gaussian components.
That is, for a dataset $\{\mathbf{x}_i\}_{i=1}^N$ generated from \emph{spherical Gaussian mixtures} with $K$ components, we assume that
\begin{align*}
h_i \sim& \text{Multi}(w_1,w_2,\dots , w_K), \\
 \mathbf{x}_i \sim& \mathcal{N}(\bm{\mu}_{h_i},\sigma^2 \mathbf{I}), \\
  i=&1,2,\dots,N
\end{align*}
where $(w_1,w_2,\dots , w_K)$ is the mixture probability, $h_i$ is the component assignment for the $i$-th data point, and $\mathcal{N}(\bm{\mu},\sigma^2 \mathbf{I})$ is a $m$-dimensional spherical Gaussian distribution with $M\geq K$. We further assume $\bm{\mu}_i \sim \mathcal{N}(\mathbf{0},\sigma_{\mu}^2 \mathbf{I})$ and $(w_1,w_2,\dots , w_K) \sim \text{Dir}(\alpha_1,\alpha_2,\dots,\alpha_K)$ for a Bayesian version of GMM. Note that we assume that the following parameters are known: $\sigma,\sigma_{\mu},\alpha_k,k=1,2,\dots,K$.

The problem on how to correctly choosing the number of mixture components has been extensively studied. Such as traditional methods (cross validation, AIC and BIC~\cite{Lukociene2010GMM}), penalized likelihood methods~\cite{Huang2013GMMPMLE} and variational approaches~\cite{Corduneanu2001GMMVariational}. Similar to the LDA model, we show that analyzing the empirical moments provides an alternative approach to bound the number of mixture components.

We define the empirical second-order moment as $\SecMH= \frac{1}{N}\sum_{i=1}^N \mathbf{x}_i \otimes \mathbf{x}_i - \sigma^2 \mathbf{I}$ and the second-order moment $\SecM$ as the expectation of the empirical moment, namely $\SecM=\Expect{\SecMH}$.
Then by similar analysis, we have the following theorem for GMM:

\begin{theorem}\label{thm:gmm:final}
Let $\alpha_k=\alpha, \forall k$, then
\begin{enumerate}[(1)]
\item Let $K_{l}$ be the number of singular values of $\SecMH$ such that $\sigma(\SecMH)>\deltaR$, where
$$
\deltaR = \frac{\sigma m}{\sqrt{ N\delta}}\sqrt{2\sigma_{\mu}^2+\frac{m+1}{m}\sigma^2},
$$
then with probability at least $1-\delta$, we have
\begin{align*}
K \geq K_l.
\end{align*}
\item Let $K_{u}$ be the maximal integer such that
\begin{align*}
\sigma_1 &(\hat{\mathbf{M}}_2) \\
 \leq   & \frac{\sigma_\mu^2}{K_u}\frac{\pq{\alpha+2\log(K_u/\delta_1)}\pq{(\sqrt{m}+\sqrt{K_u}+t)^2}}{{\max\{0_+, \alpha-\sqrt{2\alpha\log(1/\delta_2)/K_u}}\}}  + \delta_{\mathbf{R}}.
\end{align*}
Then with probability at least $1-\delta_1-\delta_2$, we have
\begin{align*}
K \leq K_u.
\end{align*}
\end{enumerate}
\end{theorem}
The proof for Theorem~\ref{thm:gmm:final} is similar to that of Theorem~\ref{thm:mul} where detailed proof is in Appendix~\ref{apd:gmm} due to space limit. As our purpose is methodology demonstration, we omit the comparison with the excellent existing works on GMM, such as \cite{steele2009performance}.


\subsection{Conclusion}

In this paper, we provide theoretical analysis for model selection in LDA. Specifically, we present both an upper bound and a lower bound on the number of topics $K$  based on the connection between second-order moments and latent topics.
The upper bound is obtained by bounding the difference between the estimated second-order moment $\hat{\mathbf{M}}_2$ and the true moment $\mathbf{M}_2$. The lower bound is obtained via analyzing the largest singular value of $\hat{\mathbf{M}}_2$. Furthermore, our analysis can be easily generalized to other latent models, such as Gaussian mixture models.

One major limitation of our approach is that all our analysis assumes that the data are generated exactly according to LDA. As a result, the analysis result may not hold when being applied to real world dataset.

For future work, we will examine effective ways to improve the theoretical results. For example, by bounding higher-order moments of $\hat{\mathbf{M}}_2-\mathbf{M}_2$ or replacing Markov inequality with tighter inequalities. Moreover, we could bound the spectral norm of $\hat{\mathbf{M}}_2-\mathbf{M}_2$ directly instead of its Frobenius norm, which potentially yields tighter bounds. 

\section{Acknowledgment}

We thank Fei Sha and David Kale for helpful discussion and suggestion. The research was
sponsored by the NSF research grants IIS-1254206, and U.S. Defense Advanced Research Projects Agency (DARPA) under Social Media in Strategic Communication (SMISC) program, Agreement Number W911NF-12-1-0034. The views and conclusions are those of the authors
and should not be interpreted as representing the official
policies of the funding agency, or the U.S. Government.

\bibliographystyle{alpha}
\bibliography{nips2013_V1}

\newpage\null\thispagestyle{empty}\newpage
\appendix
\section{Theoretical results for LDA}
\subsection{Coefficient Setting for Theorem \ref{thm:mul}}\label{apd:mul}

\subsubsection*{Bound of $\sigma_{1}(M_2)$}

We have that with probability greater than
\begin{align*}
1-&Ke^{-\frac{c^2_2}{2}}\\
-&KVe^{-\frac{c_1}{2}\min\{\frac{c_1}{2},\sqrt{\beta}\}}\\
-&K[\frac{e^{\delta'}}{(1+\delta')^{1+\delta'}}]^{\frac{V\beta}{K(\beta+c_{1}\beta^{1/2})^2}},
\end{align*}
we have
\begin{align*}
\sigma_{1}(M_2) \leq  \frac{1}{K(K\alpha+1)} \frac{(1+\delta')V(\beta+K\beta^2)}{(V\beta-c_{2}\sqrt{V\beta})^{2}}.
\end{align*}

We can choose $c_1,c_2$ and $\delta'$ as follows to simplify the formula of the bound
\begin{itemize}
\item Choose $c_{2}=\sqrt{2\log(K/\delta_1)}$, first probability term is less than $\delta_1$.
\item Choose $c_{1}=\frac{2}{\sqrt{\beta}}\log(KV/\delta_2)$, third probability term is less than $\delta_2$.
\item Choose $\delta'$ as
\begin{align*}
  \delta'=& \left(\frac{\log(K/\delta_3)K\pq{\beta+2\log\pq{K/\delta_2}}^2}{V\beta}\right)^{\frac{1}{2}},
\end{align*}
second probability term is less than $\delta_3$.
 \end{itemize}
 As a result, with probability greater than $ 1-\delta_1-\delta_2 -\delta_3
$, we have
\begin{align*}
\sigma_1(M_2) 
\leq \frac{1}{K(K\alpha+1)} \frac{(1+\delta')V(\beta+K\beta^2)}{(V\beta-\sqrt{2V\beta\log(K/\delta_1)})^{2}}.
\end{align*}

As an alternative, we can choose $c_1,c_2$ and $\delta_1$ as follows to simplify the formula of the bound
\begin{itemize}
\item Choose $c_{2}=\sqrt{2\log(K/\delta)}$, first probability term is less than $\delta$.
\item Choose $c_{1}=\frac{4}{\sqrt{\beta}}\log(KV)$, third probability term is less than $\frac{1}{KV}$.
\item Choose $\delta'=0.1$, second probability term is less than $K(0.995)^{\frac{V(\beta+K\beta^2)}{K(\beta+c_{1}\beta^{1/2})^2}}$.
 \end{itemize}
 As a result, with probability greater than 
\begin{align*}
 1-\delta-\frac{1}{KV}-K(0.995)^{\frac{V\beta}{K(\beta+2\log(KV))^2}},
\end{align*}
we have
\begin{align*}
\sigma_1(M_2) \leq \frac{1.1}{K(K\alpha+1)}\frac{V(\beta+K\beta^2)}{(V\beta-\sqrt{2V\beta\log(K/\delta) })^{2}}.
\end{align*}

\subsubsection*{Bound of $\sigma_{K}(M_2)$} 
We have that with probability greater than
\begin{align*}
1-&Ke^{-\frac{c_2}{2}\min\{\frac{c_2}{2},V\beta\}}
\\
-&KVe^{-\frac{c_1}{2}\min\{\frac{c_1}{2},\sqrt{\beta}\}}
\\
-&K[\frac{e^{-\delta'}}{(1-\delta')^{1-\delta'}}]^{\frac{V\beta}{K(\beta+c_{1}\beta^{1/2})^2}},
\end{align*}
we have
$$
\sigma_{K}(M_2) \geq  \frac{1}{K(K\alpha+1)} \frac{(1-\delta')V\beta}{(V\beta+c_{2}\sqrt{V\beta})^{2}}
$$

We can choose $c_1,c_2$ and $\delta'$ as follows to simplify the formula of the bound
\begin{itemize}
\item Choose $c_{2}=2\sqrt{\log(K/\delta_1)}$, first probability term is less than $\delta_1$.
\item Choose $c_{1}=\frac{2}{\sqrt{\beta}}\log(KV/\delta_2)$, third probability term is less than $\delta_2$.
\item Choose $\delta'$ as
\begin{align*}
  \delta'=& \left( \frac{\log(K/\delta_3)K\pq{\beta+2\log\pq{K/\delta_2}}^2}{V\beta}\right)^{\frac{1}{2}},
\end{align*}
second probability term is less than $\delta_3$.
 \end{itemize}
 As a result, with probability greater than $ 1-\delta_1-\delta_2 -\delta_3
$, we have
\begin{align*}
\sigma_K(M_2) \geq \frac{1}{K(K\alpha+1)} \frac{(1-\delta')V\beta}{(V\beta+2\sqrt{V\beta}\log(K/\delta_1))^{2}}
\end{align*}

As an alternative, we can choose $c_1,c_2$ and $\delta_1$ as follows to simplify the formula of the bound
\begin{itemize}
\item Choose $c_{1}=\frac{4}{\sqrt{\beta}}\log(KV)$, third probability term is less than $\frac{1}{KV}$.
\item Choose $c_{2}=2\sqrt{\log(K/\delta)}$, first probability term is less than $\delta$.
\item Choose $\delta'=0.1$, second probability term is less than $K(0.995)^{\frac{V(\beta+K\beta^2)}{K(\beta+c_{1}\beta^{1/2})^2}}$.
 \end{itemize}
 As a result, with probability greater than
 \begin{align*}
 1-\delta-\frac{1}{KV}-K(0.995)^{\frac{V\beta}{K(\beta+2\log(KV))^2}},
 \end{align*}
  we have
\begin{align*}
\sigma_K(M_2) \geq \frac{0.9}{K(K\alpha+1)}\frac{V\beta}{(V\beta+2\sqrt{V\beta\log{K/\delta}})^{2}}.
\end{align*}

\subsection{Lemma for Theorem \ref{thm:lda:noise}}
\begin{lemma}\label{lem:eigsvd}
With $\mathbf{\hat{M}}_2$ and $\mathbf{M}_2$ previously defined, we have that
$$
\max_i |\sigma_i(\hat{\mathbf{M}}_2) - \sigma_i(\mathbf{M}_2)| \leq \max_i |\lambda_i(\hat{\mathbf{M}}_2) - \lambda_i(\mathbf{M}_2)|
$$
\end{lemma}
\begin{proof}
Because $\mathbf{M}_2$ is a symmetric semidefinite matrix, so we have
$$
\sigma_i(\mathbf{M}_2) = \lambda_i(\mathbf{M}_2),\quad \forall i,
$$
And because $\mathbf{\hat{M}}_2$ is a symmetric matrix, we have
$$
\sigma_i(\hat{\mathbf{M}}_2) = | \lambda_{s(i)}(\hat{\mathbf{M}}_2) |,\quad \forall i,
$$
for some permutation $s$.

Because we have $\lambda_i(\hat{\mathbf{M}}_2)\leq |\lambda_i(\hat{\mathbf{M}}_2)|=\sigma_j(\hat{\mathbf{M}}_2)$, so we have $\lambda_i(\hat{\mathbf{M}}_2)\leq \sigma_i(\hat{\mathbf{M}}_2)$.

Let $j$ be the smallest index that $|\lambda_j(\hat{\mathbf{M}}_2)| \neq \sigma_j(\hat{\mathbf{M}}_2)$, for $i<j$, we have
\begin{align*}
|\sigma_i(\hat{\mathbf{M}}_2) -& \sigma_i(\mathbf{M}_2)| \\
=& |\lambda_i(\hat{\mathbf{M}}_2) - \lambda_i(\mathbf{M}_2)| \\
 \leq& \max_i |\lambda_i(\hat{\mathbf{M}}_2) - \lambda_i(\mathbf{M}_2)|
\end{align*}

By the fact that $\lambda_i(\mathbf{M}_2)\geq 0$, we have that for $\forall i\geq j$,
$$
\sigma_i(\hat{\mathbf{M}}_2) \leq  \max_k |\lambda_k(\hat{\mathbf{M}}_2) - \lambda_k(\mathbf{M}_2)|
$$
We also have
$$
\sigma_i(\hat{\mathbf{M}}_2) \geq  \lambda_i(\hat{\mathbf{M}}_2)
$$
Because
$$
|\lambda_i(\hat{\mathbf{M}}_2) - \sigma_i(\mathbf{M}_2)|  \leq \max_k |\lambda_k(\hat{\mathbf{M}}_2) - \lambda_k(\mathbf{M}_2)|
$$
We can prove that
$$
|\sigma_i(\hat{\mathbf{M}}_2) - \sigma_i(\mathbf{M}_2)| \leq \max_k |\lambda_k(\hat{\mathbf{M}}_2) - \lambda_k(\mathbf{M}_2)|
$$
Therefore,
$$
\max_i |\sigma_i(\hat{\mathbf{M}}_2) - \sigma_i(\mathbf{M}_2)| \leq \max_i |\lambda_i(\hat{\mathbf{M}}_2) - \lambda_i(\mathbf{M}_2)|
$$
\end{proof}

\section{Theoretical results for GMM}\label{apd:gmm}
The proof of Theorem~\ref{thm:gmm:final} is achieved by analyzing the concentration result $\deltaR$ of empirical second order moments and also upper bound for the first singular value of the true moment $\SecM$. Thresholding with $\deltaR$ leads to the first claim, while solving the inequality on the $\SVK{1}{\SecMH}$ provides the second claim.
\subsection{Relation Between $\mathbf{M}_2$ and $\hat{\mathbf{M}}_2$}\label{sec:gmm:M2}
 We bound the different between singular values of $\mathbf{M}_2$ through the following Theorem.
 \begin{theorem}\label{thm:gmm:noise}
For \emph{spherical Gaussian mixtures} with probability at least $1-\delta$, $\forall i\in\{1,2,\dots,m\}$,we have
\begin{equation}\label{eq:gmm:cu}
|\sigma_i(\hat{\mathbf{M}}_2) - \sigma_i(\mathbf{M}_2)| \leq \frac{\sigma m}{\sqrt{ N\delta}}\sqrt{2\sigma_{\mu}^2+\frac{m+1}{m}\sigma^2}=\delta_{\mathbf{R}}\nonumber
\end{equation}
Especially, when $i\leq K+1$, we have
\begin{equation}\label{eq:gmm:cup}
\sigma_i(\hat{\mathbf{M}}_2) \leq \frac{\sigma m}{\sqrt{ N\delta}}\sqrt{2\sigma_{\mu}^2+\frac{m+1}{m}\sigma^2}.
\end{equation}
\end{theorem}
\begin{proof}
We establish the result by bounding the Frobenius of matrix $\RM$ as we do for LDA model.
The square of Frobenius norm is $
||\mathbf{R}||_\text{F}^2 = \sum_{i,j} \mathbf{R}_{ij}^2
$. Since we have $\mathbb{E}[\mathbf{R}_{ij}|\mu]=0$, thus 
\begin{align*}
Var[\mathbf{R}_{ij}|\mu]=\mathbb{E}[\mathbf{R}_{ij}^2|\mu]-\mathbb{E}^2[\mathbf{R}_{ij}|\mu]=\mathbb{E}[\mathbf{R}_{ij}^2|\mu],
\end{align*}
and
\begin{align*}
\mathbb{E}[||\mathbf{R}||_\text{F}^2] 
= & \mathbb{E}[\mathbb{E}[||\mathbf{R}||_\text{F}^2|\mu]]  \\
= &  \mathbb{E}[\sum_{i,j}Var[\mathbf{R}_{ij}|\mu]|\mu]  \\
= & \mathbb{E}[\sum_{i\neq j}Var[\mathbf{R}_{ij}|\mu] + \sum_{i}Var[\mathbf{R}_{ii}|\mu] |\mu] \\
= & \frac{m(m-1)}{N}\sigma^2(2\sigma_{\mu}^2+\sigma^2)+\frac{m}{N}\sigma^2(2\sigma_{\mu}^2+2\sigma^2) \\
= & \frac{m^2 \sigma^2}{N}(2\sigma_{\mu}^2+\frac{m+1}{m}\sigma^2).
\end{align*}

Then by Markov inequality, we have
$$
\text{Pr}(||\mathbf{R}||_\text{F}^2 \geq k\times \mathbb{E}[||\mathbf{R}||_\text{F}^2] ) \leq 1/k.
$$

By setting $k=1/\delta$, we have that with at least probability $1-\delta$,
$$
\N{\mathbf{R}}_\text{F} \leq \frac{\sigma m}{\sqrt{ N\delta}}\sqrt{2\sigma_{\mu}^2+\frac{m+1}{m}\sigma^2}
$$
\end{proof}

\subsection{Spectral Structure of $\mathbf{M}_2$}
We use following theorem to characterize the spectral structure of $\mathbf{M}_2$.
\begin{theorem}\label{thm:gmm:spb}
Assume that $\alpha_i=\alpha$ in the \emph{spherical Gaussian mixtures}, we have

(1) With probability at least $1-\delta_1-\delta_2-2\exp(-t^2/2)$, we have
\begin{equation}\label{eq:gmm:s1}
\sigma_1(\mathbf{M}_2) \leq \frac{\sigma_\mu^2}{K}\frac{\alpha+2\log(K/\delta_1)}{\alpha-\sqrt{2\alpha\log(1/\delta_2)/K}} (\sqrt{m}+\sqrt{K}+t)^2
\end{equation}

(2) Further assume that and $w_i\geq w_{\min},\forall i$, then with probability at least $1-2\exp(-t^2/2)$, we have
\begin{equation}\label{eq:gmm:sk}
\sigma_K(\mathbf{M}_2) \geq w_{\min}\sigma_\mu^2 (\sqrt{m}-\sqrt{K}-t)^2
\end{equation}

\end{theorem}
\begin{proof}
We have $\mathbf{M}_2=\sum_{k=1}^{K}w_k \mu_k \otimes \mu_k =\mathbf{O} \mathbf{A} \mathbf{O}^{\top}$, where $\mathbf{O}=(\mu_1,\mu_2,\dots,\mu_K)$ is a $m \times K$ matrix and $\mathbf{A}=diag(w_1,w_2,\dots,w_K)$ is a diagonal matrix. Because $\mathbf{M}_2 = \mathbf{O} \mathbf{A} \mathbf{O}^{\top} = \mathbf{O}\mathbf{A}^{1/2} \mathbf{A}^{1/2}\mathbf{O}^{\top}$, we have that $\sigma_i(\mathbf{M}_2)=\sigma_i(\mathbf{A}^{1/2}\mathbf{O}^{\top}\mathbf{O}\mathbf{A}^{1/2}),\forall i=1,2,\dots,K$. Therefore, we have the following inequalities~\cite{hornmatrix}:
\begin{align}\label{eq:gmm:spb}
\sigma_1(\mathbf{M}_2) \leq \sigma_1(\mathbf{O}^{\top}\mathbf{O})\sigma_1(\mathbf{A}),
\\
\sigma_K(\mathbf{M}_2) \geq \sigma_K(\mathbf{O}^{\top}\mathbf{O})\sigma_K(\mathbf{A}).
\end{align}
Note that the elements of $\mathbf{O}$ are i.i.d. Gaussian random variables, i.e., $\mathbf{O}_{ij}\sim \mathcal{N}(0,\sigma_\mu^2)$. The distribution of $\sigma_i(\mathbf{O}^{\top}\mathbf{O})$ has been well-studied in random matrix theory~\cite{vershynin2010introduction}. With probability at least $1-2\exp(-t^2/2)$, we have
\begin{align*}
\sigma_1(\mathbf{O}^{\top}\mathbf{O}) \leq \sigma_\mu^2 (\sqrt{m}+\sqrt{K}+t)^2,
\\
\sigma_K(\mathbf{O}^{\top}\mathbf{O}) \geq \sigma_\mu^2 (\sqrt{m}-\sqrt{K}-t)^2.
\end{align*}
And since $\sigma_1(\mathbf{A})=\max_i\{w_i\}$, we can prove that with probability at least $1-\delta_1 - \delta_2$, we have (see appendix \ref{apd:sec:dir} for proof)
$$
\max_i\{w_i\} \leq \frac{1}{K}\frac{\alpha+2\log(K/\delta_1)}{\alpha-\sqrt{2\alpha\log(1/\delta_2)/K}}
$$

We also have $\sigma_K(\mathbf{A})=\min_i\{w_i\}\geq w_{\min}$. We complete the proof by substituting the above  formulas into inequalities (\ref{eq:gmm:spb}).
\end{proof}

\section{Tail bound for Gamma distribution}
In this section, we proof some tail bound related to the Gamma distribution. Our main tool is the following Lemma.
\begin{lemma}\label{lem:mnl}
\textbf{[Massart and Laurent] Tail Bound for Chi-square distribution} Let $U$ be a $\chi^2_D$ random variable with $D$ degree of freedom, then for any positive $x$, the following holds
\begin{align*}
\text{Pr}(U\geq D+2\sqrt{Dx}+2x) \leq e^{-x},
\\
\text{Pr}(U \leq D-2\sqrt{Dx}) \leq e^{-x}.
\end{align*}
\end{lemma}
\begin{proof}
See \cite{laurent2000adaptive}  for proof. 
\end{proof}

\subsection{Tail Bound for a Single Gamma Distribution}
In this section, we provide tail bound for a single Gamma random variable (R. V.).

\begin{lemma}
\textbf{Tail Bound for Gamma R.V.} Let $X\sim Gamma(\alpha,1)$ be a Gamma R.V. with shape parameter $\alpha$, and scale parameter $1$, then for any positive $c$, the following holds
\begin{align*}
\text{Pr}(X \geq \alpha+c\sqrt{\alpha}) \leq & e^{-\frac{c}{2}\min\{\frac{c}{2},\sqrt{\alpha}\}},
\\
\text{Pr}(X \leq \alpha - c\sqrt{\alpha}) \leq & e^{-\frac{c^2}{2}}.
\end{align*}

\end{lemma}
\begin{proof}
By relationship between Gamma R.V. and chi-square R.V., we have that $2X\sim \chi_{2\alpha}^2$. Apply Lemma \ref{lem:mnl} directly, we have
\begin{align*}
\text{Pr}(X \geq \alpha+c\sqrt{\alpha}) \leq & e^{-c\sqrt{\alpha}+\alpha(\sqrt{1+2c\alpha^{-1/2}}-1)},
\\
\text{Pr}(X \leq \alpha - c\sqrt{\alpha}) \leq & e^{-\frac{c^2}{2}}.
\end{align*}

To get the same formula as in the lemma, we can easily prove that $c\sqrt{\alpha}-\alpha(\sqrt{1+2c\alpha^{-1/2}}-1)> \frac{c}{2}\min\{\frac{c}{2},\sqrt{\alpha}\},\quad \forall c,\alpha>0 $.

\end{proof}

\begin{corollary}\label{crl:SSGRV}
\textbf{Tail Bound for Sum of Square of Gamma R.V.} If we have $n$ i.i.d Gamma R.V. $X_i\sim Gamma(\alpha,1),i=1,\dots,n$, then for any positive $c$, the following holds
$$
\text{Pr}(\sum_i X_i^2 \geq n(\alpha+c\sqrt{\alpha})^2 ) \leq n e^{-\frac{c}{2}\min\{\frac{c}{2},\sqrt{\alpha}\}}.
$$
\end{corollary}

\subsection{Tail Bound for Maximum/Minimum of Gamma Random Variables}
\begin{lemma}\label{lem:GammaMaxMin}
If we have $n$ i.i.d Gamma R.V. $X_i\sim Gamma(\alpha,1),i=1,\dots,n$, we have that
\begin{align*}
\text{Pr}(\max_i \{X_i\} \geq \alpha+c\sqrt{\alpha}) \leq & n e^{-\frac{c}{2}\min\{\frac{c}{2},\sqrt{\alpha}\}},
\\
\text{Pr}(\min_i \{X_i\} \leq \alpha - c\sqrt{\alpha}) \leq & n e^{-\frac{c^2}{2}}.
\end{align*}
\end{lemma}
\begin{proof}
It can be proved by applying union bound directly. 
\end{proof}

\subsection{Tail Bound for Maximum/Minimum Element of Dirichlet Distribution}
\label{apd:sec:dir}

It is well known that a random vector $(x_1,x_2,\dots,x_n) \sim \text{Dir}(\alpha_1,\alpha_2,\dots,\alpha_n)$ is equivalent to a random vector $(y_1,y_2,\dots,y_n)/\sum_i y_i$, where $y_i \sim \text{Gamma}(\alpha_i,1)$ independently. And we have $\max_i\{x_i\}=\max_i\{y_i\}/\sum_i y_i$.

Assume $\alpha_i=\alpha$, so we have
$$
\text{Pr}(\max_i \{y_i\} \geq \alpha+c_1\sqrt{\alpha}) \leq n e^{-\frac{c_1}{2}\min\{\frac{c_1}{2},\sqrt{\alpha}\}}.
$$

And since $\sum_i y_i \sim \text{Gamma}(n\alpha,1)$, we have
$$
\text{Pr}(\sum_i y_i \leq n\alpha - c_2\sqrt{n\alpha}) \leq e^{-\frac{c_2^2}{2}}
$$

By setting $c_1=2\log(n/\delta_1)/\sqrt{\alpha}$ (when $n>\delta_1 e^{\alpha}$) and $c_2 = \sqrt{2\log(1/\delta_2)}$, we have that with probability at least $1-\delta_1-\delta_2$,
$$
\max_i\{x_i\} \leq \frac{1}{n}\frac{\alpha+\log(n/\delta_1)}{\alpha - \sqrt{2\alpha\log(1/\delta_2)/n}}
$$

Similarity, $\min_i\{x_i\}=\min_i\{y_i\}/\sum_i y_i$. And
\begin{align*}
\text{Pr}(\min_i \{x_i\} \leq \alpha - c_1 \sqrt{\alpha}) \leq & n e^{-\frac{c_1^2}{2}},
\\
\text{Pr}(\sum_i y_i \geq n\alpha + c_2\sqrt{n\alpha}) \leq & e^{-\frac{c_2}{2}\min\{\frac{c_2}{2},\sqrt{n\alpha}\}}.
\end{align*}

By setting $c_1 = \sqrt{2\log(n/\delta_1)}$ and $ c_2 = \sqrt{2\log(1/\delta_2)}$ (when $\delta_2 > e^(-2\alpha)$), we have that with probability at least $1-\delta_1-\delta_2$,
$$
\min_i\{x_i\} \geq \frac{1}{n}   \frac{   \alpha - \sqrt{2\log(n\alpha/\delta_1)} }{  \alpha + \sqrt{2\alpha\log(1/\delta_2)/n}    }
$$
which is nontrivial only when $\alpha$ is large enough.

\section{Variance Calculation for LDA}\label{apd:lda}

In this section, we presents the overall procedure and some important intermediate results of the variance calculation for LDA. Note that we have the following assumptions on the scale of each statistics or parameters: $L=\mathcal{O}(D)$, $V=\mathcal{O}(D)$, $L=\mathcal{O}(V)$, $K=\mathcal{O}(L)$, $1/K=\mathcal{O}(1)$, $\alpha=\Theta(1)$, and $\beta=\Theta(1)$.

First, we have
\begin{align*}
R= & \frac{1}{D}\sum_{d}\frac{1}{L(L-1)}\sum_{l\neq s}x_{d,l}x_{d,s}^{\top}\\
- & \frac{\alpha_{0}}{\alpha_{0}+1}[\frac{1}{D}\sum_{d}\frac{1}{L}\sum_{l}x_{d,l}][\frac{1}{D}\sum_{d}\frac{1}{L}\sum_{l}x_{d,l}]^{\top}\\
-&M_{2}.
\end{align*}

We represent each term by
\begin{align*}
R^{(1)}=&\frac{1}{D}\sum_{d}\frac{1}{L(L-1)}\sum_{l\neq s}x_{d,l}x_{d,s}^{\top},
\\
R^{(2)}=&\frac{\alpha_{0}}{\alpha_{0}+1}[\frac{1}{D}\sum_{d}\frac{1}{L}\sum_{l}x_{d,l}][\frac{1}{D}\sum_{d}\frac{1}{L}\sum_{l}x_{d,l}]^{\top},
\\
R^{(3)}=&\frac{1}{D}\sum_{d}\frac{1}{L}\sum_{l}x_{d,l}.
\end{align*}

And we have the following identity:
\begin{align*}
E_{\mu}Var_{X}[R_{ij}]= & E_{\mu}Var_{X}[R_{ij}^{(1)}]+E_{\mu}Var_{X}[R_{ij}^{(2)}]\\
- &2 E_{\mu}Cov_{X}[R_{ij}^{(1)},R_{ij}^{(2)}],
\end{align*}
with $H=\{\mu,h\}$, $X=\{h,x\}$.
\begin{align*}
R_{ij}^{(2)}=\frac{\alpha_{0}}{\alpha_{0}+1}R_{i}^{(3)}R_{j}^{(3)}.
\end{align*}

For simplicity of representation, we assume the following,
\begin{align*}
f_{d}^{(ij)}= &\frac{1}{L(L-1)}\sum_{l\neq s}^{L}x_{d,l}^{(i)}x_{d,s}^{(j)},
\\
g_{d}^{(i)}= & \frac{1}{L}\sum_{l=1}^{L}x_{d,l}^{(i)}.
\end{align*}
and the superscript $(ij)$ or $(i)$ will be omitted if there is
no ambiguity. By this representation, we have
\begin{align*}
R^{(1)}=&\frac{1}{D}\sum_{d}f_{d},
\\
R^{(3)}=&\frac{1}{D}\sum_{d}g_{d}.
\end{align*}

We also assume the representation $z_{d}^{(i)}=\sum_{k}\mu_{k}^{(i)}h_{d}^{(k)}$,
which is the probability of $e_{i}$ in the $d$-th documents conditioned
on $H=\{\mu,h\}$. And $\delta_{ij}=1$ if and only if $i=j$.

The intermediate results for diagonal and off-diagonal variance are different, so we provide them separately in the following sections.

\subsection{Calculate Off-diagonal Variance}
In this section, we assume that $i\neq j$. And we have the following results:
\begin{align*}
E_{\mu}Var_{X}[R_{ij}^{(1)}]
\leq& \frac{1}{DL^{2}V^{2}}+\frac{2}{DLV^{3}}+\frac{1}{DV^{4}}+O(\epsilon)
\\
E_{\mu}Var_{X}[R_{ij}^{(2)}]
\leq& \frac{2}{DLV^{3}} + \frac{1}{DV^{4}}+O(\epsilon)
\\
E_{\mu}Cov_{X}(R^{(1)},R^{(2)}) \geq& \frac{2}{DLV^{3}} +O(\epsilon)
\end{align*}

Therefore, we have that
\begin{align*}
E_{\mu}Var_{X}[R_{ij}] \leq & \frac{1}{DL^{2}V^{2}}+\frac{2}{DLV^{3}}+\frac{1}{DV^{4}} +  \frac{2}{DLV^{3}} \\
+& \frac{1}{DV^{4}} - \frac{4}{DLV^{3}} +O(\epsilon) \\
= & \frac{1}{DL^{2}V^{2}} + \frac{2}{DV^{4}} + O(\epsilon).
\end{align*}

\subsection{Calculate Diagonal Variance}

In this section, we assume that $i\neq j$. And we have the following results:
\begin{align*}
E_{\mu}Var_{X}[R_{ij}^{(1)}]
\leq & \frac{1}{DL^{2}V}+\frac{4}{DLV^{3}}+\frac{1}{DV^{4}}+O(\epsilon),
\\
E_{\mu}Var_{X}[R_{ij}^{(2)}]
\leq& \frac{2}{DLV^{3}}  + \frac{1}{DV^{4}} + O(\epsilon),
\\
E_{\mu}Cov_{X}(R^{(1)},R^{(2)}) \geq & \frac{3}{DLV^{3}} +O(\epsilon).
\end{align*}

Therefore, we have that
\begin{align*}
E_{\mu}Var_{X}[R_{ij}] \leq & \frac{1}{DL^{2}V}+\frac{4}{DLV^{3}}+\frac{1}{DV^{4}}+  \frac{2}{DLV^{3}} \\
 + &\frac{1}{DV^{4}} - \frac{6}{DLV^{3}} +O(\epsilon) \\
= & \frac{1}{DL^{2}V} + \frac{2}{DV^{4}} + O(\epsilon).
\end{align*}

\end{document}